\documentclass{article}
\usepackage{nips12submit_e,times}
\usepackage{graphicx}
\usepackage{epstopdf}
\usepackage{amsmath,amssymb}
\usepackage{ mathrsfs }
\usepackage[numbers,sort&compress]{natbib}
\usepackage[noend]{algorithmic}
\usepackage{algorithm}
\usepackage{multirow}
\usepackage{url}
\usepackage{hyperref}

\usepackage{amsfonts}
\usepackage[amsmath,amsthm,thmmarks]{ntheorem}
\usepackage{multicol}
\usepackage{appendix}

\usepackage[subrefformat=parens,labelformat=parens]{subfig}

\newcommand{\Dset}{\mathcal{D}}
\newcommand{\Dspace}{\mathbb{D}}
\newcommand{\R}{\mathbb{R}}

\newcommand{\F}{\mathcal{F}}

\renewcommand{\L}{{\mathcal L}}

\DeclareMathOperator*{\argmin}{argmin}

\newtheorem{theorem}{Theorem}
\theoremstyle{plain}

\newtheorem{lemma}{Lemma}

\nipsfinalcopy

\begin{document}

\title{Multi-criteria Anomaly Detection using \\
Pareto Depth Analysis}
\author{
Ko-Jen Hsiao, Kevin S.~Xu, Jeff Calder, and Alfred O.~Hero III\\
University of Michigan, Ann Arbor, MI, USA 48109\\
\texttt{\{coolmark,xukevin,jcalder,hero\}@umich.edu}
}

\maketitle

\begin{abstract}

We consider the problem of identifying patterns in a data set that exhibit 
anomalous behavior, often referred to as anomaly detection. 
In most anomaly detection algorithms, the dissimilarity between data samples 
is calculated by a single criterion, such as Euclidean distance. 
However, in many cases there may not exist a single dissimilarity measure that captures all possible anomalous patterns. In such a case, multiple criteria can be defined, and one can test for anomalies by scalarizing the multiple criteria using a linear combination of them.
If the importance of the different criteria are not known in 
advance, the algorithm may need to be executed multiple times with 
different choices of weights in the linear combination. 
In this paper, we introduce a novel non-parametric \emph{multi-criteria} 
anomaly detection method using \emph{Pareto depth analysis} (PDA). 
PDA uses the concept of Pareto optimality to detect anomalies under 
multiple criteria without having to 
run an algorithm multiple times with different choices of weights. 
The proposed PDA approach scales \emph{linearly} in the number of criteria 
and is provably better than linear combinations of the criteria.  

\end{abstract}

\section{Introduction}
\label{intro}
Anomaly detection is an important problem that has been studied in a 
variety of areas and used in diverse applications including intrusion 
detection, fraud detection, and image processing 
\citep{Hodge2004,Chandola2009}. 
Many methods for anomaly detection have been developed using both parametric 
and non-parametric approaches. 
Non-parametric approaches typically involve the calculation of dissimilarities 
between data samples. 
For complex high-dimensional data, multiple dissimilarity measures 
corresponding to different criteria 
may be required to detect certain types of anomalies. 
For example, consider the problem of detecting anomalous object trajectories 
in video sequences. 
Multiple criteria, such as dissimilarity in object speeds or trajectory 
shapes, can be used to detect a greater range of anomalies than any single 
criterion. 
In order to perform anomaly detection using these multiple criteria, one 
could first combine the dissimilarities using a linear combination. 
However, in many applications, the importance of the criteria are 
not known in advance. 
It is difficult to determine how much weight to assign to each dissimilarity 
measure, so one may have to choose multiple weights using, for 
example, a grid search. 
Furthermore, when the weights are changed, the anomaly detection algorithm 
needs to be re-executed using the new weights. 

In this paper we propose a novel non-parametric 
\emph{multi-criteria} anomaly detection approach using 
\emph{Pareto depth analysis} (PDA). 
PDA uses the concept of Pareto optimality to detect anomalies without having 
to choose weights for different criteria. 
Pareto optimality is the typical method for defining optimality when there may 
be multiple conflicting criteria for comparing items. 
An item is said to be Pareto-optimal if there does not exist another item that 
is better or equal in all of the criteria. 
An item that is Pareto-optimal is optimal in the usual sense 
under some combination, not necessarily linear, of the criteria. 
Hence, PDA is able to detect anomalies under multiple 
combinations of the criteria without explicitly forming these 
combinations. 

The PDA approach involves creating \emph{dyads} corresponding to 
dissimilarities 
between pairs of data samples under all of the dissimilarity measures. 
Sets of Pareto-optimal dyads, called \emph{Pareto fronts}, are then computed. The first Pareto front (depth one) is the set of non-dominated dyads. The second Pareto front (depth two) is obtained by removing these non-dominated dyads, i.e.~peeling off the first front, and recomputing the first Pareto front of those remaining. This process continues until no dyads remain. In this way, each dyad is assigned to a Pareto front at some depth (see  Figure \ref{bi_sample} for illustration).
Nominal and anomalous samples are located near different Pareto 
front depths; thus computing the front depths of the dyads corresponding to a test 
sample can discriminate between nominal and anomalous 
samples. 
The proposed PDA approach scales \emph{linearly} 
in the number of criteria, which 
is a significant improvement compared to selecting multiple weights via a 
grid search, which scales exponentially in the number of criteria. Under assumptions that the multi-criteria dyads can be modeled as a realizations from a smooth $K$-dimensional density we provide a mathematical analysis of the behavior of the first Pareto front. This analysis shows in a precise sense that PDA can outperform a test that uses a linear combination of the criteria.
Furthermore, this theoretical prediction is experimentally validated by comparing PDA to several state-of-the-art anomaly detection algorithms in two experiments involving both synthetic and 
real data sets.

\begin{figure}[tp]
\begin{center}
  \parbox{\textwidth}{
  \includegraphics[width=4.6 cm]{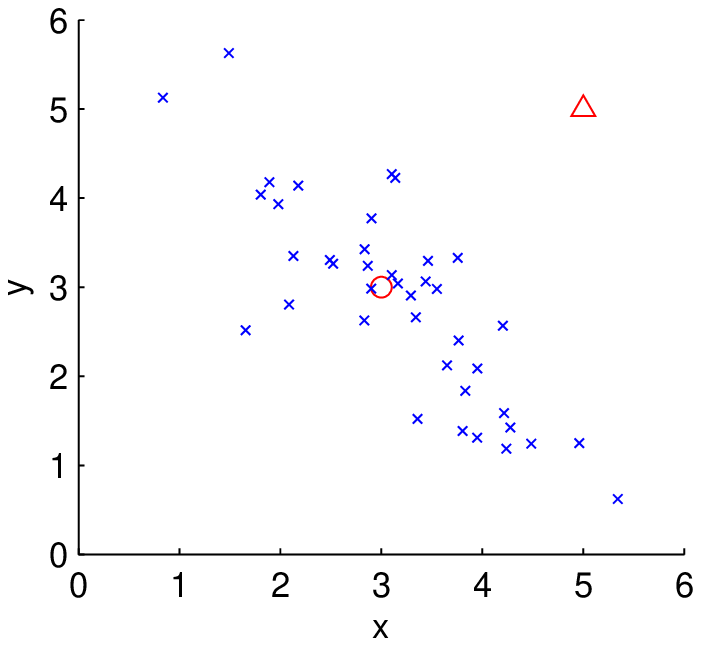}
  \includegraphics[width=4.6 cm]{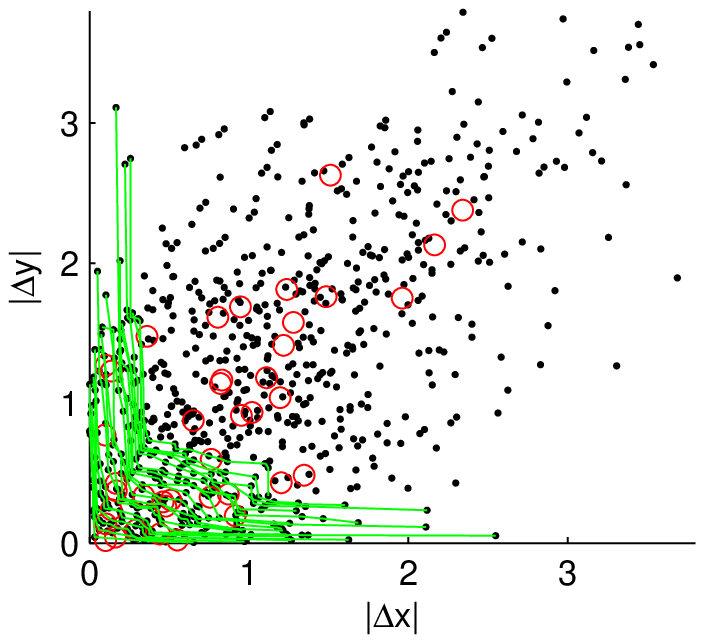}
  \includegraphics[width=4.6 cm]{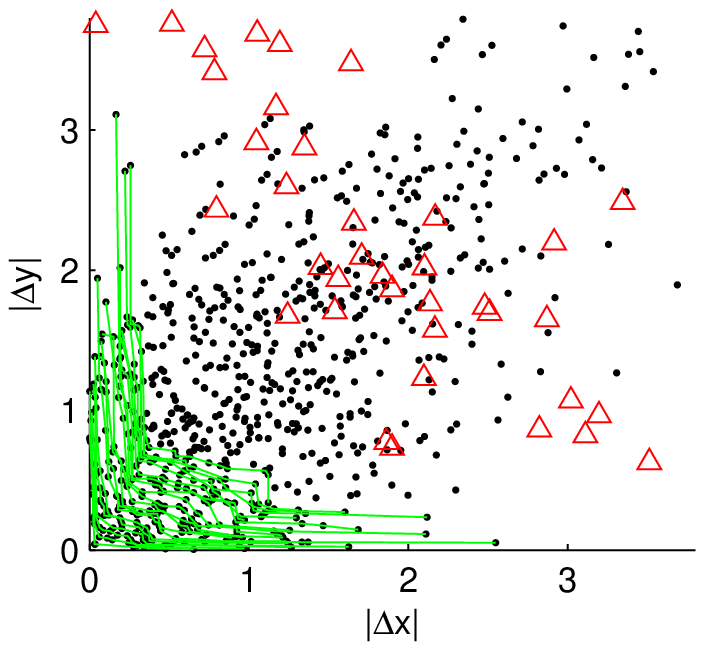}

  }
  \caption{\emph{Left:} Illustrative example with $40$ training samples 
  (blue x's) and $2$ test samples (red circle and triangle) in $\R^2$. 
  \emph{Center:} Dyads for the training samples (black dots) along with first 
  $20$  Pareto fronts (green lines) under two criteria: $|\Delta x|$ and 
  $|\Delta y|$. 
  The Pareto fronts induce a partial ordering on the set of dyads. 
  Dyads associated with the test sample marked by the red circle concentrate around 
  shallow fronts (near the lower left of the figure). 
  \emph{Right:} Dyads associated with the test sample marked by the red triangle 
  concentrate around deep fronts.}
  \label{bi_sample}
  \end{center}
\end{figure}

The rest of this paper is organized as follows. 
We discuss related work in Section \ref{related}. 
In Section \ref{prealg} we provide an introduction to Pareto fronts and present a 
theoretical analysis of the properties of the first Pareto front. 
Section \ref{detection} relates Pareto fronts to the multi-criteria 
anomaly detection problem, which leads to the PDA anomaly detection algorithm. Finally we present two experiments in Section \ref{exp} to evaluate the 
performance of PDA. 

\section{Related work}
\label{related}
Several machine learning methods utilizing Pareto optimality have previously 
been proposed; an overview can be found in \citep{Jin2008}. 
These methods typically formulate machine learning problems as 
multi-objective optimization problems where finding even the first Pareto 
front is quite difficult. 
These methods differ from our use of Pareto optimality because we consider \emph{multiple} Pareto fronts created from 
a finite set of items, so 
we do not need to employ sophisticated methods in order to find these 
fronts.

\citet{pareto1} introduced a method for gene ranking using Pareto 
fronts that is related to our approach. 
The method ranks genes, in order of interest to a 
biologist, by creating Pareto fronts of the data samples, i.e.~the genes. 
In this paper, we consider Pareto fronts of \emph{dyads}, which correspond to 
dissimilarities between \emph{pairs} of data samples rather than the samples 
themselves, and use the distribution 
of dyads in Pareto fronts to perform multi-criteria anomaly detection rather than ranking.  

Another related area is multi-view learning \citep{Blum1998,Sindhwani2005}, 
which involves learning from data represented by multiple sets of features, 
commonly referred to as ``views''. 
In such case, training in one view helps to improve learning in another view. 
The problem of view disagreement, where samples take different classes in 
different views, has recently been investigated \citep{Christoudias2008}. 
The views are similar to criteria in our problem setting. 
However, in our setting, different criteria may be orthogonal and could even 
give contradictory information; hence there may be severe view disagreement. 
Thus training in one view could actually worsen performance 
in another view, so the problem we consider differs from multi-view learning. 
A similar area is that of multiple kernel learning \citep{Gonen2011}, 
which is typically 
applied to supervised learning problems, unlike the unsupervised anomaly 
detection setting we consider.

Finally, many other anomaly detection methods have previously been proposed.
\citet{Hodge2004} and \citet{Chandola2009} both provide extensive surveys of 
different anomaly detection methods and applications. 
Nearest neighbor-based methods are closely related to the proposed 
PDA approach. 
\citet{Byers1998} proposed to use the distance between a sample and its 
$k$th-nearest neighbor as the anomaly score for the sample; similarly, 
\citet{Angiulli2002} and \citet{Eskin2002} proposed to the use the sum of the 
distances between a sample and its $k$ nearest neighbors. 
\citet{Breunig2000} used an anomaly score based on the local density of the 
$k$ nearest neighbors of a sample. 
\citet{Hero2006} and \citet{Sricharan2011} introduced non-parametric adaptive anomaly detection methods using geometric entropy minimization, based on random $k$-point 
minimal spanning trees and bipartite $k$-nearest neighbor ($k$-NN) graphs, respectively. 
\citet{Zhao2009} proposed an anomaly detection algorithm k-LPE using local 
p-value estimation (LPE) based on a $k$-NN graph. 
These $k$-NN anomaly detection schemes only depend on the data through the pairs of data points (dyads) that define the edges in the $k$-NN graphs.

All of the aforementioned methods are designed for \emph{single-criteria} anomaly 
detection. 
In the \emph{multi-criteria} setting, the single-criteria algorithms must be executed 
multiple times with different weights, unlike the PDA anomaly detection 
algorithm that we propose in Section \ref{detection}.

\section{Pareto depth analysis}
\label{prealg}

The PDA method proposed in this paper utilizes the notion of Pareto optimality, 
which has been studied in many application areas in economics, computer science, 
and the social sciences among others \citep{ehrgott}. 
We introduce Pareto optimality and define the notion of a 
Pareto front. 

Consider the following problem: given $n$ items, denoted by the set $\mathcal{S}$, 
and $K$ criteria for evaluating each item, denoted by functions $f_1, \ldots, 
f_K$, select $x \in \mathcal{S}$ that minimizes $[f_1(x), \ldots, f_K(x)]$. 
In most settings, it is not possible to identify a single item $x$ that 
simultaneously minimizes $f_i(x)$ for all $i \in \{1, \ldots, K\}$. 
A minimizer can be found by combining the $K$ criteria using a linear 
combination of the $f_i$'s and finding the minimum of the combination. 
Different choices of (non-negative) weights in the linear combination 
could result in 
different minimizers; a set of items that are minimizers under some linear 
combination can then be created by using a grid search over the weights, 
for example. 

A more powerful approach involves finding the set of Pareto-optimal items. 
An item $x$ is said to \emph{strictly dominate} another item $x^*$ if 
$x$ is no greater than $x^*$ in each criterion and 
$x$ is less than $x^*$ in at least one criterion. 
This relation can be written as $x\succ x^*$ if $f_i(x) \leq f_i(x^*)$ for 
each $i$ and $f_i(x) < f_i(x^*)$ for some $i$. 
The set of Pareto-optimal items, called the \emph{Pareto front}, is the 
set of items in $\mathcal{S}$ that are not strictly dominated by another item 
in $\mathcal{S}$. 
It contains all of the minimizers that are found using linear combinations, 
but also includes other items that cannot be found by linear 
combinations.
Denote the Pareto front by $\F_1$, which we call the first Pareto front. 
The second Pareto front can be constructed by finding items that are 
not strictly dominated by any of the remaining items, which are members of 
the set $\mathcal{S} \setminus \F_1$. 
More generally, define the $i$th Pareto front by
$$\F_i=\text{Pareto front of the set }
\mathcal{S} \setminus \left(\bigcup_{j=1}^{i-1}\F_j\right).$$
For convenience, we say that a Pareto front $\F_i$ is \emph{deeper} 
than $\F_j$ if $i>j$. 

\subsection{Mathematical properties of Pareto fronts}
\label{theory}
The distribution of the number of points on the first Pareto front was first studied by Barndorff-Nielsen and Sobel in their seminal work~\cite{nielsen1966}.  The problem has garnered much attention since; for a survey of recent results see~\cite{bai2005}.  We will be concerned here with properties of the first Pareto front that are relevant to the PDA anomaly detection algorithm and thus have not yet been considered in the literature.  Let $Y_1,\dots,Y_n$ be independent and identically distributed~(i.i.d.)~on $\R^d$ with density function $f: \R^d \to \R$. For a measurable set $A \subset \R^d$, we denote by $\F_A$ the points on the first Pareto front of $Y_1,\dots,Y_n$ that belong to $A$.  For simplicity, we will denote $\F_1$ by $\F$ and use $|\F|$ for the cardinality of $\F$.  In the general Pareto framework, the points $Y_1,\dots,Y_n$ are the images in $\R^d$ of $n$ feasible solutions to some optimization problem under a vector of objective functions of length $d$. 
In the context of this paper, each point $Y_l$ corresponds to a dyad $D_{ij}$, which we define in Section \ref{detection}, and $d=K$ is the number of criteria.  
A common approach in multi-objective optimization is linear scalarization~\citep{ehrgott}, which constructs a new single criterion as a convex combination of the $d$ criteria.  It is well-known, and easy to see, that linear scalarization will only identify Pareto points on the boundary of the convex hull of $\bigcup_{x \in \F} (x + \R^d_+)$, where $\R^d_+ = \{ x \in \R^d \, | \, x_i\geq 0, i=1\,\dots,d\}$.
Although this is a common motivation for Pareto methods, there are, to the best of our knowledge, no results in the literature regarding how many points on the Pareto front are missed by scalarization. We present such a result here. We define
\[\L = \bigcup_{\alpha \in \R^d_+} \argmin_{x \in S_n} \left\{\sum_{i=1}^d \alpha_i x_i \right\}, \ \ S_n =\{Y_1,\dots,Y_n\}. \]
The subset $\L \subset \F$ contains all Pareto-optimal points that can be obtained by some selection of weights for linear scalarization.  We aim to study how large $\L$ can get, compared to $\F$, in expectation.  In the context of this paper, if some Pareto-optimal points are not identified, then the anomaly score~(defined in section \ref{score})~will be artificially inflated, making it more likely that a non-anomalous sample will be rejected. Hence the size of $\F\setminus\L$ is a measure of how much the anomaly score is inflated and the degree to which Pareto methods will outperform linear scalarization.  

Pareto points in $\F\setminus \L$ are a result of non-convexities in the Pareto front.  We study two kinds of non-convexities: those induced by the geometry of the domain of $Y_1,\dots,Y_n$, and those induced by randomness. We first consider the geometry of the domain. Let $\Omega \subset \R^d$ be bounded and open with a smooth boundary $\partial  \Omega$ and suppose the density $f$ vanishes outside of $\Omega$.  For a point $z \in \partial \Omega$ we denote by $\nu(z)=(\nu_1(z),\dots,\nu_d(z))$ the unit inward normal to $\partial \Omega$.
For $T \subset \partial \Omega$, define $T_h\subset \Omega$ by $T_h = \{ z + t\nu \, | \, z \in T, 0 < t \leq h\}$.
Given $h>0$ it is not hard to see that all Pareto-optimal points will almost surely lie in $\partial\Omega_h$ for large enough $n$, provided the density $f$ is strictly positive on $\partial \Omega_h$.  Hence it is enough to study the asymptotics for $E|\F_{T_h}|$ for $T \subset \partial \Omega$ and $h>0$.
\begin{theorem}\label{thm:asym}
Let $f \in C^1(\overline{\Omega})$ with $\inf_\Omega f >0$.  Let $T \subset \partial \Omega$ be open and connected such that
\[\inf_{z\in T} \min(\nu_1(z),\dots,\nu_d(z)) \geq \delta > 0, \quad \text{and} \quad \{y \in \overline{\Omega} \, : \, y \preceq x\} = \{x\}, \ \ {\rm for} \ \ x \in T.\]
Then for $h>0$ sufficiently small, we have 
\[E|\F_{T_h} | = \gamma n^{\frac{d-1}{d}} + \delta^{-d-1}O\left(n^{\frac{d-2}{d}}\right) \ \ {\rm as} \ \ n \to \infty,\]
where
$\displaystyle \gamma = d^{-1}(d!)^{\frac{1}{d}}\Gamma(d^{-1})\int_T f(z)^{\frac{d-1}{d}} (\nu_1(z)\cdots \nu_d(z))^{\frac{1}{d}}  dz$.
\end{theorem}
The proof of Theorem \ref{thm:asym} is postponed to Appendix \ref{proofs}.
Theorem \ref{thm:asym} shows asymptotically how many Pareto points are contributed on average by the segment $T \subset \partial \Omega$.  The number of points contributed depends only on the geometry of $\partial \Omega$ through the direction of its normal vector $\nu$ and is otherwise independent of the convexity of $\partial \Omega$.  Hence, by using Pareto methods, we will identify significantly more Pareto-optimal points than linear scalarization when the geometry of $\partial \Omega$ includes non-convex regions.  For example, if $T \subset\partial \Omega$ is non-convex~(see left panel of Figure \ref{fig:theory}) and satisfies the hypotheses of Theorem \ref{thm:asym}, then for large enough $n$, all Pareto points in a neighborhood of $T$ will be unattainable by scalarization.  Quantitatively, if $f \geq C$ on $T$, then $E|\F\setminus \L| \geq  \gamma n^\frac{d-1}{d} + \delta^{-d-1} O(n^\frac{d-2}{d})$, as $n \to \infty$, where $\gamma \geq d^{-1}(d!)^\frac{1}{d} \Gamma(d^{-1})|T|\delta C^\frac{d-1}{d}$ 
and $|T|$ is the $d-1$ dimensional Hausdorff measure of $T$.  It has recently come to our attention that Theorem \ref{thm:asym}  appears in a more general form in an unpublished manuscript of \citet{baryshnikov2005}.

We now study non-convexities in the Pareto front which occur due to inherent randomness in the samples.  We show that, even in the case where $\Omega$ is convex, there are still numerous small-scale non-convexities in the Pareto front that can only be detected by Pareto methods. We illustrate this in the case of the Pareto box problem for $d=2$. 
\begin{theorem}\label{thm:small}
Let $Y_1,\dots,Y_n$ be independent and uniformly distributed on $[0,1]^2$.  Then
\[\frac{1}{2}\ln n + O(1) \leq E|\L| \leq  \frac{5}{6} \ln n + O(1), \ \ {\rm as} \ \ n\to \infty.\]
\end{theorem}
\begin{figure}[tp]
\begin{center}
   \mbox{
  \includegraphics[height=0.16\textheight]{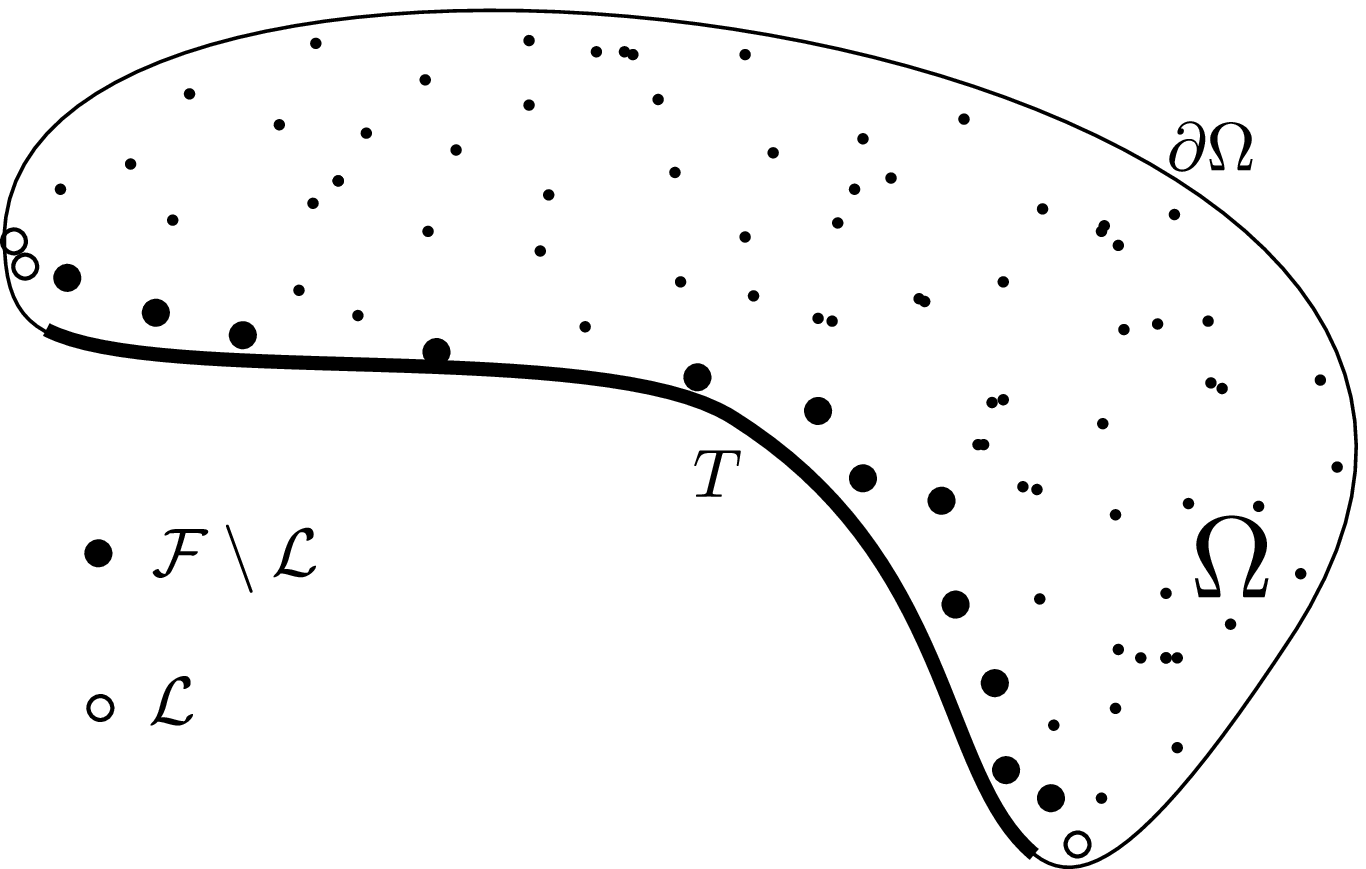} 
  \includegraphics[trim = 25 17 25 25, clip=true,height=0.16\textheight]{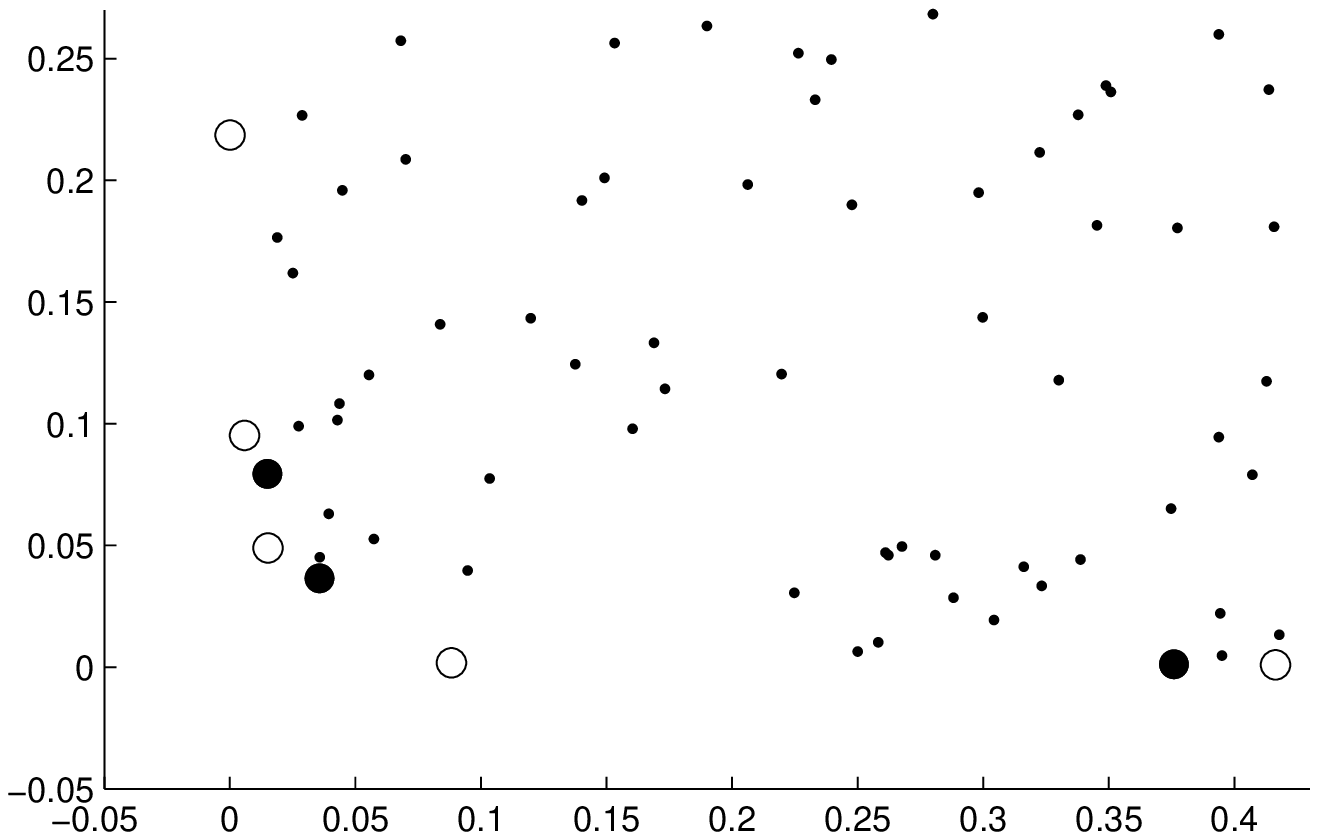}
  }
  \caption{\emph{Left:} Non-convexities in the Pareto front induced by the geometry of the domain $\Omega$ (Theorem \ref{thm:asym}). \emph{Right:} Non-convexities due to randomness in the samples (Theorem \ref{thm:small}). In each case, the larger points are Pareto-optimal, and the large black points \emph{cannot} be obtained by scalarization.  }
  \label{fig:theory}
  \end{center}
\vspace{-0.5cm}
\end{figure}
The proof of Theorem \ref{thm:small} is also postponed to Appendix \ref{proofs}.  A proof that $E|\F| = \ln n + O(1)$ as $n\to \infty$ can be found in~\cite{nielsen1966}.  Hence Theorem \ref{thm:small} shows that, asymptotically and in expectation, only between $\frac{1}{2}$ and $\frac{5}{6}$ of the Pareto-optimal points can be obtained by linear scalarization in the Pareto box problem.  Experimentally, we have observed that the true fraction of points is close to $0.7$.   This means that at least $\frac{1}{6}$ (and likely more) of the Pareto points can only be obtained via Pareto methods even when $\Omega$ is convex.  
Figure \ref{fig:theory} gives an example of the sets $\F$ and $\L$ from the two theorems.  

\section{Multi-criteria anomaly detection}
\label{detection}

Assume that a training set $\mathcal{X}_N=\{X_1,\ldots,X_N\}$ of 
nominal data samples is available. 
Given a test sample $X$, the objective of anomaly detection is to declare $X$ to 
be an anomaly if $X$ is significantly different from samples in $\mathcal{X}_N$. 
Suppose that $K>1$ different evaluation criteria are given. 
Each criterion is associated with a measure for computing 
dissimilarities.  
Denote the dissimilarity between $X_i$ and $X_j$ computed using the measure 
corresponding to the $l$th criterion by $d_l(i,j)$. 

We define a \emph{dyad} by $D_{ij}=[d_1(i,j), \ldots, d_K(i,j)]^T \in 
\mathbb{R}_{+}^K, i \in \{1,\ldots,N\}, j \in \{1,\ldots,N\} \setminus i$. 
Each dyad $D_{ij}$ corresponds to a connection between samples $X_i$ and $X_j$. 
Therefore, there are in total ${N\choose 2}$ different dyads. 
For convenience, denote the set of all dyads by $\Dset$ and the space 
of all dyads $\mathbb{R}^K_+$ by $\Dspace$.
By the definition of strict dominance in Section \ref{prealg}, a dyad 
$D_{ij}$ 
strictly dominates another dyad $D_{i^* j^*}$ if $d_l(i,j) \leq d_l(i^*,j^*)$ 
for all $l \in \{1, 
\ldots, K\}$ and $d_l(i,j) < d_l(i^*,j^*)$ for some $l$. 
The first Pareto front $\F_1$ corresponds to the set of dyads from $\Dset$ that 
are not strictly dominated by any other dyads from $\Dset$. 
The second Pareto front $\F_2$ corresponds to the set of dyads from 
$\Dset \setminus \F_1$ that are not strictly dominated by any other dyads from 
$\Dset \setminus \F_1$, and so on, as defined in Section \ref{prealg}. 
Recall that we refer to $\F_i$ as a \emph{deeper} front than $\F_j$ if $i>j$. 

\subsection{Pareto fronts of dyads}
\label{PDA_anomaly}
For each sample $X_n$, there are $N-1$ dyads corresponding to its connections 
with the other $N-1$ samples. 
Define the set of $N-1$ dyads associated with $X_n$ by $\Dset^n$. 
If most dyads in $\Dset^n$ are located at shallow Pareto fronts, then the 
dissimilarities between $X_n$ and the other $N-1$ samples are small under 
\emph{some} combination of the criteria. 
Thus, $X_n$ is likely to be a nominal sample. 
This is the basic idea of the proposed multi-criteria anomaly detection 
method using PDA. 

We construct Pareto fronts $\F_1,\ldots,\F_M$ of the dyads from the training 
set, where the total number of fronts $M$ 
is the required number of fronts such that each dyad is a member of a front. 
When a test sample $X$ is obtained, we create new dyads corresponding to 
connections between $X$ and training samples, as illustrated in Figure \ref{bi_sample}. 
Similar to many other anomaly detection methods, 
we connect each test sample to its $k$ nearest neighbors. 
$k$ could be different for each criterion, so we denote $k_i$ as the choice 
of $k$ for criterion $i$. 
We create $s = \sum_{i=1}^K k_i$ new dyads, which we denote by the set 
$\Dset^\text{new} = 
\{D_1^\text{new}, D_2^\text{new}, \ldots, D_s^\text{new}\}$, corresponding to 
the connections between $X$ and the union of the $k_i$ nearest neighbors in each 
criterion $i$. 
In other words, we create a dyad between $X$ and $X_j$ if $X_j$ is among the 
$k_i$ nearest neighbors\footnote{If a training sample is one of the $k_i$ nearest neighbors 
in multiple criteria, then multiple copies of the dyad corresponding to the connection 
between the test sample and the training sample are created.} of $X$ in any criterion $i$. 
We say that $D_i^\text{new}$ is \emph{below} a front $\F_l$ if 
$D_i^\text{new} \succ D_l \text{ for some } D_l \in \F_l$,
i.e.~$D_i^\text{new}$ strictly dominates at least a single dyad in $\F_l$. 
Define the depth of $D_i^{\text{new}}$ by 
\begin{equation*}
e_i=\min\{l\,|\,D_i^\text{new} \text{ is below } \F_l\}.
\end{equation*}
Therefore if $e_i$ is large, then $D^\text{new}_i$ will be near deep fronts, and the distance between $X$ and the corresponding training sample is 
large under \emph{all} combinations of the $K$ criteria. 
If $e_i$ is small, then $D^\text{new}_i$ will be near shallow fronts, so the 
distance between $X$ and the corresponding training sample is small under 
\emph{some} combination of the $K$ criteria. 

\subsection{Anomaly detection using depths of dyads}
\label{score}
In k-NN based anomaly detection algorithms such as those mentioned in 
Section \ref{related}, the \emph{anomaly score} is a function of the $k$ 
nearest neighbors to a test sample. 
With multiple criteria, one could define an anomaly score by 
scalarization. 
From the probabilistic properties of Pareto fronts discussed in Section \ref{theory}, we know that Pareto methods identify more Pareto-optimal points than linear scalarization methods and significantly more Pareto-optimal points than a single weight for scalarization\footnote{Theorems \ref{thm:asym} and \ref{thm:small} require i.i.d.~samples, but dyads are not independent. However, there are $O(N^2)$ dyads, and each dyad is only dependent on $O(N)$ other dyads.  This suggests that the theorems should also hold for the non-i.i.d.~dyads as well, and it is supported by experimental results presented in Appendix \ref{exp_sup}.}. 

This motivates us to develop a \emph{multi-criteria anomaly score} using Pareto fronts.
We start with the observation from Figure \ref{bi_sample} that dyads 
corresponding to a nominal test sample are typically located near shallower
fronts than dyads corresponding to an anomalous test sample. 
Each test sample is associated with $s$ new dyads, where the $i$th dyad 
$D_i^\text{new}$ has depth $e_i$. 
For each test sample $X$, we define the anomaly score $v(X)$ to be the mean of the $e_i$'s, which corresponds to the average depth of the $s$ dyads associated with $X$. Thus the anomaly score can be easily computed and compared to the decision 
threshold $\sigma$ using the test
$$
v(X) = \frac{1}{s} \sum_{i=1}^s e_i \overset{H_1}{\underset{H_0}{\gtrless}} \sigma.
$$

\begin{algorithm}[t]
\caption{PDA anomaly detection algorithm.}
\label{alg}
Training phase:
\begin{algorithmic}[1]
	\FOR {$l = 1 \to K$}
		\STATE {Calculate pairwise dissimilarities $d_l(i,j)$ between all 
		training samples $X_i$ and $X_j$}
	\ENDFOR
	\STATE {Create dyads $D_{ij} = [d_1(i,j), \ldots, d_K(i,j)]$ for all 
	training samples}
	\STATE {Construct Pareto fronts on set of all dyads until each dyad is 
	in a front}
\end{algorithmic}
Testing phase:
\begin{algorithmic}[1]
	\STATE {$nb \leftarrow [\,]$} \COMMENT{empty list}
	\FOR {$l = 1 \to K$}
		\STATE {Calculate dissimilarities between test sample $X$ 
		and all training samples in criterion $l$}
		\STATE {$nb_l \leftarrow k_l$ nearest neighbors of $X$}
		\STATE {$nb \leftarrow [nb,nb_l]$} \COMMENT{append neighbors to list}
	\ENDFOR
	\STATE {Create $s$ new dyads $D_i^\text{new}$ between $X$ and training 
	samples in $nb$}
	\FOR {$i = 1 \to s$}
		\STATE {Calculate depth $e_i$ of $D_i^\text{new}$}
	\ENDFOR
	\STATE {Declare $X$ an anomaly if $v(X) = (1/s) \sum_{i=1}^s e_i > \sigma$}	
\end{algorithmic}
\end{algorithm}

Pseudocode for the PDA anomaly detector is shown in Algorithm 
\ref{alg}. 
In Appendix \ref{implement} we provide details of the implementation as well as an analysis of the time complexity and a heuristic for choosing the $k_i$'s that performs well in practice.
Both the training time and the time required to test a 
new sample using PDA are \emph{linear} in the number of criteria $K$. 
To handle multiple criteria, other anomaly detection methods, such as the 
ones mentioned in Section \ref{related}, need to be re-executed multiple 
times using different (non-negative) linear combinations of the $K$ criteria. 
If a grid search is used for selection of the weights in the linear 
combination, then the required computation time would be exponential in $K$. 
Such an approach presents a computational problem unless $K$ is very small. 
Since PDA scales \emph{linearly} with $K$, it does not encounter this problem. 

\section{Experiments}
\label{exp}
We compare the PDA method with four other nearest neighbor-based 
single-criterion anomaly 
detection algorithms mentioned in Section \ref{related}. 
For these methods, we use linear combinations of the criteria with 
different weights selected by grid search to compare performance with PDA. 

\subsection{Simulated data with four criteria}
First we present an experiment on a simulated data set.
The nominal distribution is given by the uniform distribution on 
the hypercube $[0,1]^4$. 
The anomalous samples are located just outside of this hypercube. 
There are four classes of anomalous distributions. 
Each class differs from 
the nominal distribution in one of the four dimensions; the distribution 
in the anomalous dimension is uniform on $[1,1.1]$. 
We draw $300$ training samples from the nominal distribution followed by 
$100$ test samples from a mixture of the nominal and anomalous 
distributions with a $0.05$ probability of selecting any particular 
anomalous distribution. 
The four criteria for this experiment correspond to the squared differences 
in each dimension. 
If the criteria are combined using linear combinations, 
the combined dissimilarity measure reduces to weighted squared Euclidean 
distance.

The different methods are evaluated using the receiver operating 
characteristic (ROC) curve and the area under the curve (AUC). 
The mean AUCs (with standard errors) over $100$ simulation runs are shown 
in Table \subref*{auc_table_sim}. 
A grid of six points between $0$ and $1$ in each criterion, 
corresponding to $6^4=1296$ different sets of weights, is used to 
select linear combinations for the single-criterion methods. 
Note that PDA is the best performer, outperforming even the best linear combination. 

\begin{table}[t]
\renewcommand{\arraystretch}{1.1}
\caption{AUC comparison of different methods for both experiments.
Best AUC is shown in {\bf bold}. 
PDA does not require selecting weights so it has a single AUC. 
The median and best AUCs (over all choices of weights selected by grid search) are 
shown for the other four methods. 
\emph{PDA outperforms all of the other methods, even for the best 
weights, which are not known in advance.}} 
\begin{center}
\subfloat[Four-criteria simulation ($\pm$ standard error)]{
\label{auc_table_sim}
\begin{tabular}{ccc}
\hline
\multirow{2}{*}{Method} & \multicolumn{2}{c}{AUC by weight}\\
& Median & Best\\
\hline
PDA & \multicolumn{2}{c}{\bf{0.948 $\pm$ 0.002}}\\
k-NN       &0.848 $\pm$ 0.004&    0.919 $\pm$ 0.003\\
k-NN sum   &0.854 $\pm$ 0.003&    0.916 $\pm$ 0.003\\
k-LPE     &0.847 $\pm$ 0.004&    0.919 $\pm$ 0.003\\
LOF       &0.845 $\pm$ 0.003&    0.932 $\pm$ 0.003\\
\hline
\end{tabular}
}
\qquad\qquad
\subfloat[Pedestrian trajectories]{
\label{auc_table}
\begin{tabular}{ccc}
\hline
\multirow{2}{*}{Method}& \multicolumn{2}{c}{AUC by weight}\\
 & Median & Best\\
\hline
PDA & \multicolumn{2}{c}{\bf{0.915}}\\
k-NN       &0.883&    0.906\\
k-NN sum   &0.894&    0.911\\
k-LPE     &0.893&    0.908\\
LOF       &0.839&    0.863\\
\hline
\end{tabular}
}
\end{center}
\end{table}

\subsection{Pedestrian trajectories}
\label{traj}
We now present an experiment on a real data set that contains thousands of 
pedestrians' trajectories in an open area monitored by a video camera \citep{Majecka2009}. 
Each trajectory is approximated by a cubic spline curve 
with seven control points \citep{Sillito2009}. 
We represent a trajectory with $l$ time samples by
$$T = \begin{bmatrix}
x_1 & x_2 & \ldots & x_l \\
y_1 & y_2 & \ldots & y_l
\end{bmatrix},$$
where $[x_t,y_t]$ denote a pedestrian's position at time step $t$. 

We use two criteria for computing the dissimilarity between trajectories. 
The first criterion is to compute the dissimilarity in \emph{walking speed}. 
We compute the instantaneous speed at all time steps along each trajectory 
by finite differencing, i.e.~the speed of trajectory $T$ at time step $t$ 
is given by $\sqrt{(x_t - x_{t-1})^2 + (y_t - y_{t-1})^2}$. 
A histogram of speeds for each trajectory is obtained in this manner. 
We take the dissimilarity between two trajectories to be the squared 
Euclidean distance between their speed histograms. 
The second criterion is to compute the dissimilarity in \emph{shape}. 
For each trajectory, we select $100$ points, uniformly positioned along the 
trajectory. 
The dissimilarity between two trajectories $T$ and $T'$ is then given by the 
sum of squared Euclidean distances between the positions of $T$ and $T'$ 
over all $100$ points. 

\begin{figure}[t]
\begin{center}
  \mbox{
  \includegraphics[width=6.8 cm]{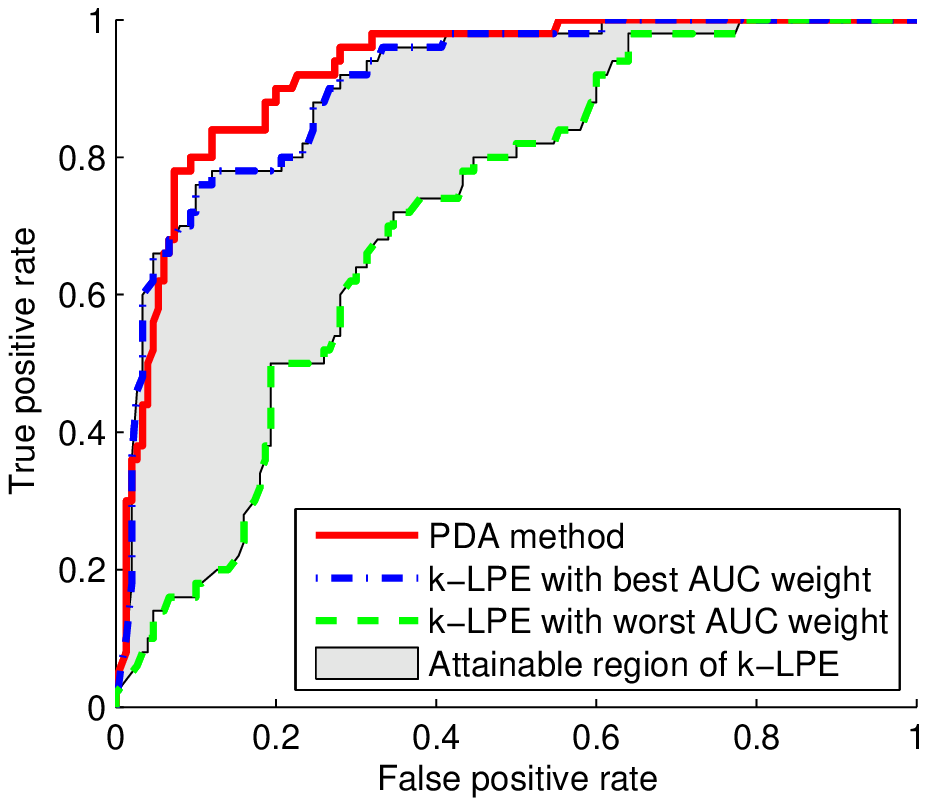}
  \includegraphics[width=6.8 cm]{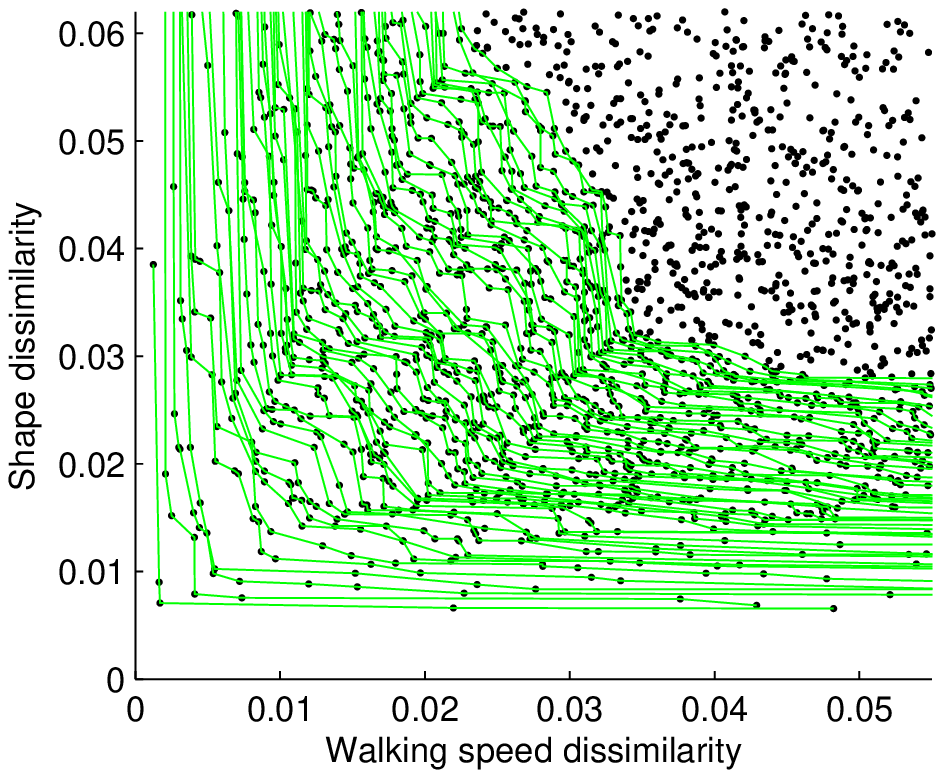}
  }
  \caption{\emph{Left:} ROC curves for PDA and attainable region for k-LPE over 
  $100$ choices of weights. 
  PDA outperforms k-LPE even under the best choice of weights. \emph{Right}: A subset of the dyads for the training samples along with the first $100$ Pareto fronts. The fronts are highly non-convex, partially explaining the superior performance of PDA.}
  \label{roc_curves}
  \end{center}
\end{figure}

The training sample for this experiment consists 
of $500$ trajectories, and the test sample consists of $200$ trajectories. 
Table \subref*{auc_table} shows the performance of PDA as compared to the other 
algorithms using $100$ uniformly spaced weights for linear combinations. 
Notice that PDA has higher AUC than the other methods under \emph{all} choices of 
weights for the two criteria. 
For a more detailed comparison, the ROC curve for PDA and the 
attainable region for k-LPE (the region between the ROC curves corresponding 
to weights resulting in the best and worst AUCs) is shown in Figure \ref{roc_curves} along with the first $100$ Pareto fronts for PDA. 
k-LPE performs slightly better at low false positive rate when the 
best weights are used, but PDA performs better in all other situations, resulting 
in higher AUC. 
Additional discussion on this experiment can be found in Appendix 
\ref{traj_add}. 

\section{Conclusion}
\label{con}
In this paper we proposed a new multi-criteria anomaly detection method. 
The proposed method uses Pareto depth analysis to compute the anomaly score of 
a test sample by examining the Pareto front depths of dyads corresponding to the 
test sample. 
Dyads corresponding to an anomalous sample tended to be located at deeper fronts 
compared to dyads corresponding to a nominal sample. 
Instead of choosing a specific weighting or performing a grid search on the 
weights for different dissimilarity measures, the proposed method can 
efficiently detect anomalies in a manner that scales linearly in the 
number of criteria. 
We also provided a theorem establishing that the Pareto approach is asymptotically better than using linear combinations of criteria. Numerical studies validated our theoretical predictions of PDA's performance advantages on simulated and real data.

\subsubsection*{Acknowledgments}
We thank Zhaoshi Meng for his assistance in labeling the pedestrian 
trajectories.
We also thank Daniel DeWoskin for suggesting a fast algorithm for computing 
Pareto fronts in two criteria.
This work was supported in part by ARO grant W911NF-09-1-0310.

\appendixtitleon
\appendixtitletocon
\begin{appendices}

\section{Proofs of Theorems \ref{thm:asym} and \ref{thm:small}}
\label{proofs}

Before presenting the proofs of Theorems \ref{thm:asym} and \ref{thm:small} we need a preliminary result.

\begin{lemma}\label{lem:num_pts}
For any $n\geq 1$ and $A \subset \R^d$ measurable, we have
\begin{equation}\label{eq:num_pts}
E |\F_A| = n \int_A f(x) \left(1-\int_{y \preceq x} f(y) dy\right)^{n-1} dx.
\end{equation}
\end{lemma}
\begin{proof}
Since $Y_1,\dots,Y_n$ are i.i.d, we have $E|\F_A| = nP(Y_1 \in \F)$.
Conditioning on $Y_1$ we obtain $E|\F_A| =n \int_{\R^d} f(x) P(Y_1 \in \F \, | \, Y_1=x) dx$.
The proof is completed by noting that
\[P(Y_1 \in \F \, | \, Y_1=x) = \begin{cases}
\left(1-\int_{y \preceq x} f(y) dy\right)^{n-1},& x \in A,\\
0,& x \not\in A. 
\end{cases} \]
\end{proof}

\begin{proof}[Proof of Theorem \ref{thm:asym}]
By selecting $h>0$ smaller, if necessary, we can write \eqref{eq:num_pts} as
\begin{equation}\label{eq:lema}
E|\F_{T_h}| = \int_{T} \int_0^h nf(x)\left(1-\int_{y \preceq x} f dy \right)^{n-1}(1+ O(t)) dt dz,
\end{equation}
where $x = z + t\nu(z)$ for $z \in T$. Since $\partial \Omega$ is smooth, we can approximate $T$ near $z$ by a hyperplane with normal $\nu(z)$.  By the assumption that $\{y \in \overline{\Omega} \, : \, y \preceq x\} = \{x\}$ we can make $h>0$ smaller, if neceessary, so that $\{ y \in \Omega \, : \, y \preceq x\}$ is approximately a simplex with side lengths $t/\nu_i(z)$.  Hence
\begin{eqnarray*}
\int_{y \preceq x} f(y) dy &=& (f(z) + O(t/\delta))\int_{y \preceq x} dy \\
&=& \frac{f(z)t^d}{d!\nu_1(z)\cdots\nu_d(z)} + O\left(\frac{t^{d+1}}{\delta^{d+1}}\right).
\end{eqnarray*}
Substituting this into \eqref{eq:lema}, we have
\begin{equation}\label{eq:lemb}
E|\F_{T_h}| = \int_T\int_0^h n(f(z) + O(t))\left(1-\frac{f(z)t^d}{d!\nu_1(z)\cdots\nu_d(z)} + O(t^{d+1}/\delta^{d+1})\right)^{n-1} dt dz.
\end{equation}
We can now do an asymptotic analysis of the inner integral which is a special case of the general equation
\[A_n := \int_0^h t^\lambda (1-at^d + O(bt^{d+1}))^{n-1} dt, \ \ \lambda \in [0,1], a,b>0. \]
Making the change of variables $-s = (n-1)\ln(1-at^d + O(bt^{d+1}))$ and simplifying, we obtain
\[A_n = \frac{1}{(a(n-1))^\frac{1+\lambda}{d}} \int_0^{P(n-1)} \left(\frac{1}{d} s^{\frac{1+\lambda}{d} - 1} + \frac{b}{(n-1)^\frac{1}{d}} O(s^{\frac{2+\lambda}{d} -1}) \right) e^{-s} ds,\]
where
\[P = -\ln(1-ah^d + bO(h^{d+1})).\]
We can, of course, choose $h$ small enough so that $P$ is finite and positive.
Recalling the definition of the Gamma function, $\Gamma(z) = \int_0^\infty t^{z-1} e^{-t} dt$, we see that
\[A_n = \frac{\Gamma\left(\frac{1+\lambda}{d}\right)}{d (an)^\frac{1+\lambda}{d}} + O\left(\frac{b}{n^\frac{2+\lambda}{d}}\right). \]
Note that we are keeping track of $O(b)$ terms because $b=O(1/\delta^{d+1})$ may become large at different points of $T$, whereas $O(1/a)$ is uniformly bounded independent of $\delta$ along $T$. 
Applying this to \eqref{eq:lemb} with 
\[a = \frac{f(z)}{d!\nu_1(z)\cdots\nu_d(z)}, \ \ {\rm and} \ \ b = \delta^{-(d+1)},\]
completes the proof.
\end{proof}
\vspace{6pt}

\begin{proof}[Proof of Theorem \ref{thm:small}]
Since $Y_1,\dots,Y_n$ are i.i.d., we have $E|\L| = nP(Y_1 \in \L)$.  For $(x,y) \in [0,1]^2$ let $D_{x,y}$ be the event that $Y_1=(x,y)$ and $(x,y) \in \F$.  Conditioning on $D_{x,y}$ we have
\begin{align}\label{eq:small}
E|\L| &{}={} n\int_0^1 \int_0^1 (1-xy)^{n-1} P((x,y) \in \L \, | \, D_{x,y})\, dxdy \notag \\
&{}={} n\int_0^\frac{1}{2} \int_0^\frac{1}{2}(1-xy)^{n-1} P((x,y) \in \L \, | \, D_{x,y})\, dxdy + O(1).
\end{align}
Define
\[A = \left\{(u,v) \in [0,1]^2 \; | \; 0 <u<x, \;  y < v < 2y -\frac{uy}{x}\right\},\]
and
\[B = \left\{(u,v) \in [0,1]^2 \; | \; x < u <1, \;  0 < v < 2y -\frac{uy}{x}\right\}.\]
Let $E$ be the event that $A$ and $B$ each contain at least one sample from $Y_2,\dots,Y_n$.  If $E$ occurs, then $(x,y)$ is in the interior of the convex hull of $\F$ and hence $(x,y) \not\in \L$.  Let $F$ denote the event that none of the samples from $Y_2,\dots,Y_n$ fall in $A\cup B$.  If $F$ occurs, then we clearly have $(x,y) \in \L$.  It follows that
\[P(F \, | \, D_{x,y}) \leq P((x,y) \in \L \, | \, D_{x,y}) \leq P(E^c \, | \,D_{x,y}).\]
Conditioned on $D_{x,y}$, the samples $Y_2,\dots, Y_n$ remain independent.  The conditional density function of each remaining sample is $f_{Y_i\, | \, D_{x,y}}(u,v) = \frac{1}{1-xy}$.  Let $E_A$ (resp.~$E_B$) denote the event that no samples from $Y_2,\dots,Y_n$ are drawn from $A$ (resp.~$B$).  Then $E^c = E_A\cup E_B$ and $F = E_A \cap E_B$.  Noting that $|A|=|B|=\frac{1}{2}xy$, we see that
\begin{eqnarray*}
P(E^c \, | \, D_{x,y}) &=& P(E_A \, | \, D_{x,y})+ P(E_B \, | \, D_{x,y}) - P(E_A\cap E_B \, | \, D_{x,y}) \\
&=&2\left( 1 - \frac{xy}{2(1-xy)}\right)^{n-1} - \left(1-\frac{xy}{1-xy}\right)^{n-1},
\end{eqnarray*}
and
\[P(F \, | \, D_{x,y}) = P(E_A \cap E_B \, | \, D_{x,y}) = \left(1-\frac{xy}{1-xy}\right)^{n-1}.\]
Substituting this into \eqref{eq:small}, we obtain
\[E|\L| \leq n\int_0^\frac{1}{2} \int_0^\frac{1}{2} 2\left(1-\frac{3}{2}xy\right)^{n-1} - (1-2xy)^{n-1} \, dxdy,\]
and
\[E|\L| \geq n\int_0^\frac{1}{2} \int_0^\frac{1}{2} (1-2xy)^{n-1} \, dxdy.\]
A short calculation (change variables to $u=anxy$ and $v=x$) shows that
\[\int_0^\frac{1}{2} \int_0^\frac{1}{2}n(1-axy)^{n-1} dxdy = \frac{1}{a} \ln n + O(1).\]
Applying this result to the bounds above completes the proof.
\end{proof}

\section{Experimental support for Theorems \ref{thm:asym} and \ref{thm:small}}
\label{exp_sup}

Independence of $Y_1,\dots,Y_n$ is built into the assumptions of Theorems \ref{thm:asym} and \ref{thm:small}, but it is clear that dyads (as constructed in Section \ref{detection}) are not independent.  Each dyad $D_{i,j}$ represents a connection between two independent samples $X_i$ and $X_j$.  For a given dyad $D_{i,j}$, there are $2(N-2)$ corresponding dyads involving $X_i$ or $X_j$ and these are clearly not independent from $D_{i,j}$.  However, all other dyads are independent from $D_{i,j}$.  So while there are $O(N^2)$ dyads, each dyad is independent from all other dyads except for a set of size $O(N)$.  Since Theorems \ref{thm:asym} and \ref{thm:small} deal with asymptotic results, this suggests they should hold for the dyads even though they are not i.i.d.  In this section we present some experimental results that support this non-rigorous statement.

\begin{figure}[t]
\centering
\subfloat[Criteria $|\Delta x|$,$|\Delta y|$]{
\includegraphics[width=0.45\textwidth]{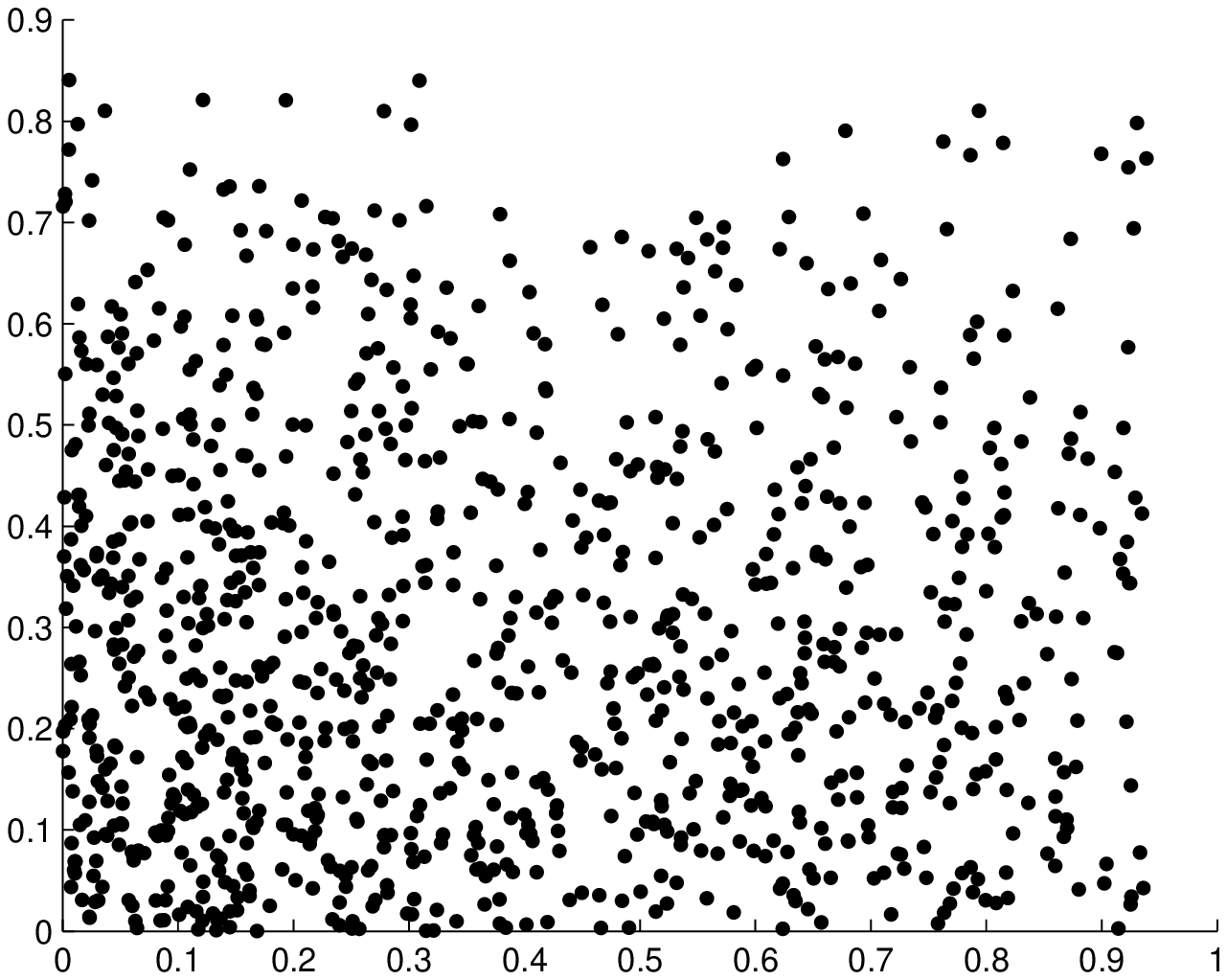}
\label{dyad_samplesa}
}
\subfloat[Criteria $|\Delta x| + |\Delta y|$,$|\Delta x| - |\Delta y|$]{
\includegraphics[width=0.45\textwidth]{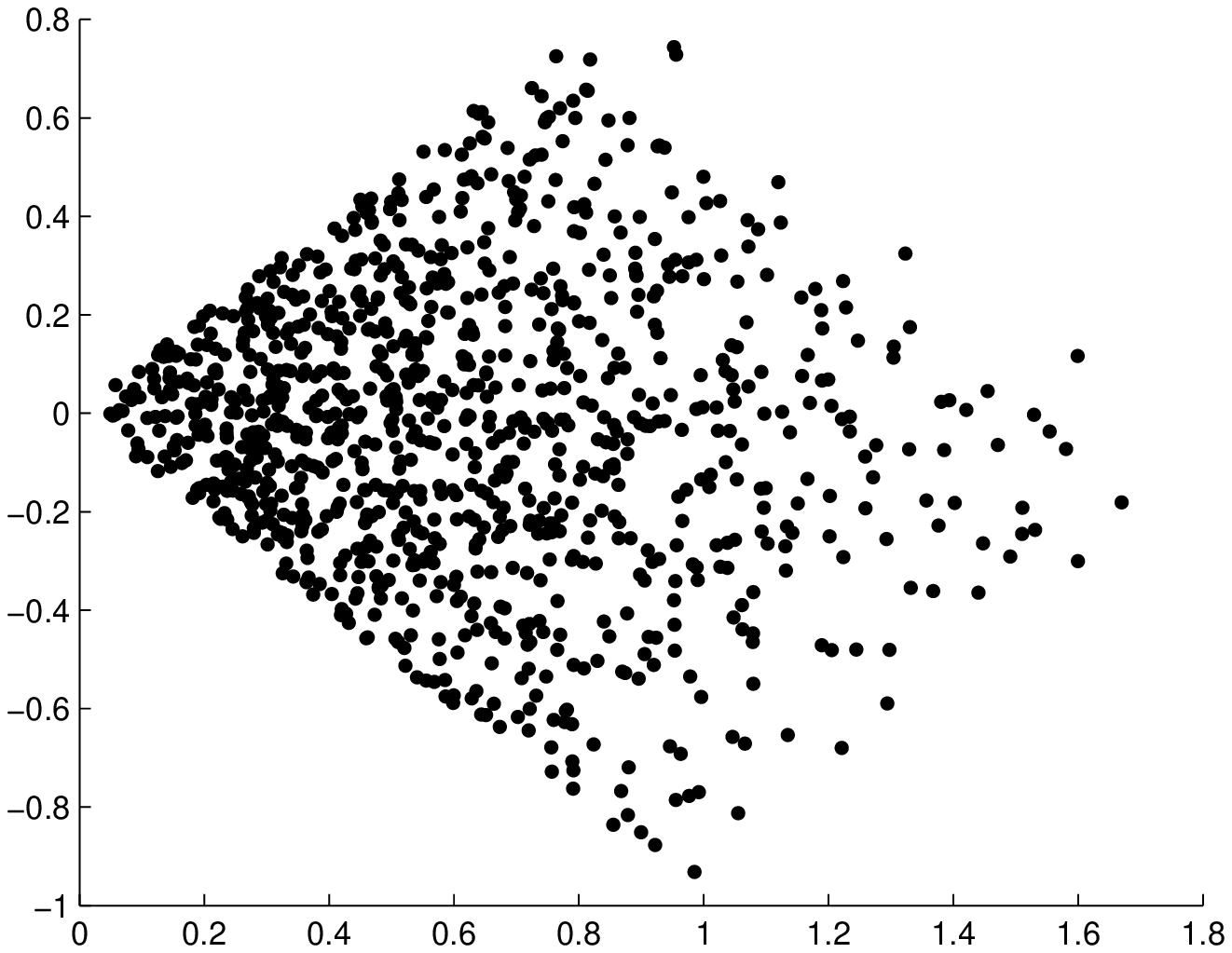}
\label{dyad_samplesb}
}
\caption{$990$ dyads constructed with two different sets of criteria from $45$ samples uniformly distributed in $[0,1]^2$.} \label{dyad_samples}
\end{figure}

We first drew samples uniformly in $[0,1]^2$ and computed the dyads corresponding to the two criteria $|\Delta x|$ and $|\Delta y|$, which denote the absolute differences between the $x$ and $y$ coordinates, respectively.  The domain of the resulting dyads is again the box $[0,1]^2$, as shown in Figure \subref*{dyad_samplesa}, so this experiment tests Theorem \ref{thm:small}.  In this case, Theorem \ref{thm:small} suggests that $\F \setminus \L$ should grow logarithmically.  Figure \subref*{dyad_dataa} shows the sample means versus number of dyads and a best fit logarithmic curve of the form $y = \alpha \ln n$, where $n = {N \choose 2}$ denotes the number of dyads.  A linear regression on $y/\ln n$ versus $\ln n$ gave $\alpha =0.3142$ which falls in the range specified by Theorem \ref{thm:small}.

\begin{figure}[t]
\centering
\subfloat[Criteria $|\Delta x|$,$|\Delta y|$]{
\includegraphics[width=0.45\textwidth]{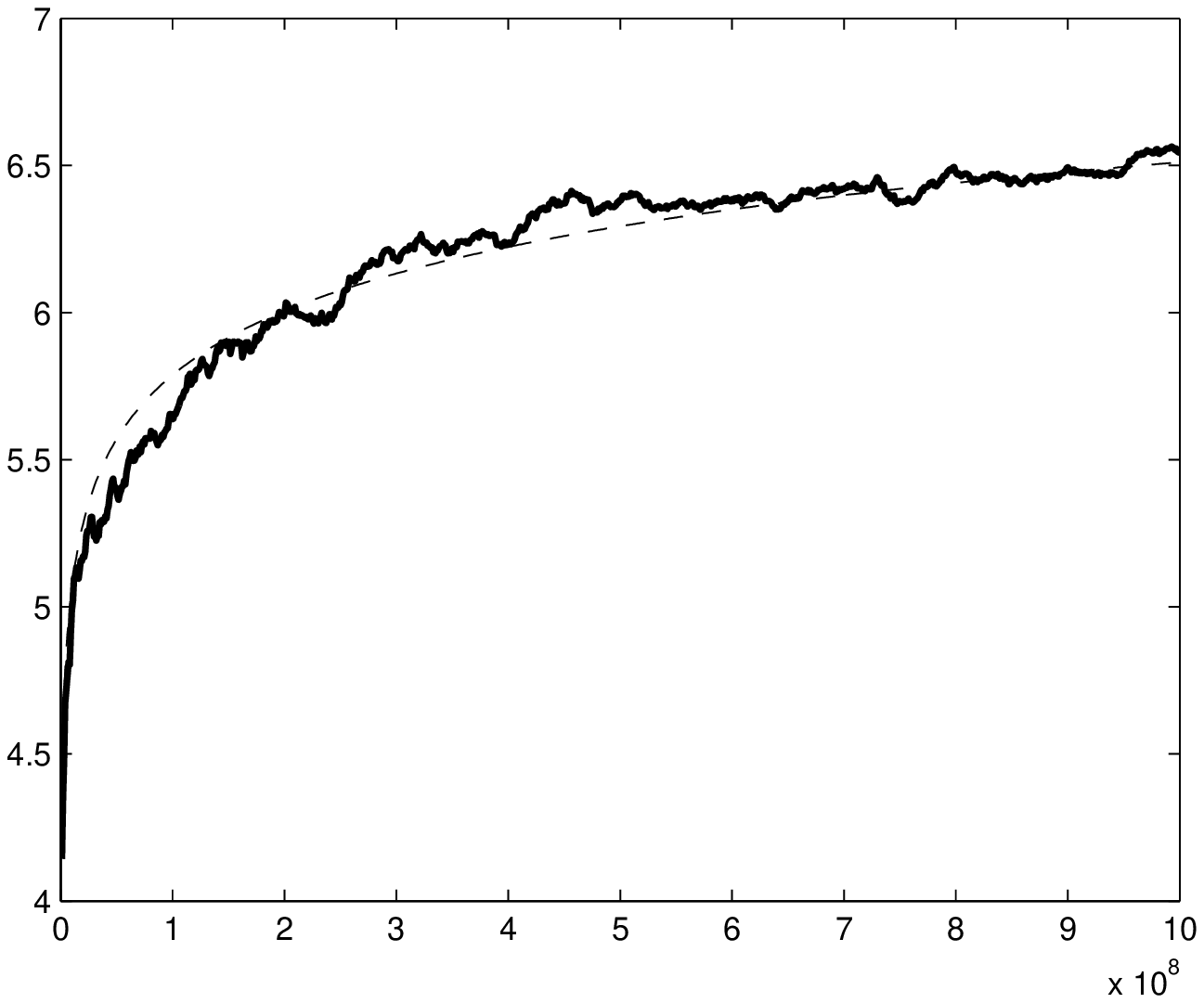}
\label{dyad_dataa}
}
\subfloat[Criteria $|\Delta x| + |\Delta y|$,$|\Delta x| - |\Delta y|$]{
\includegraphics[width=0.45\textwidth]{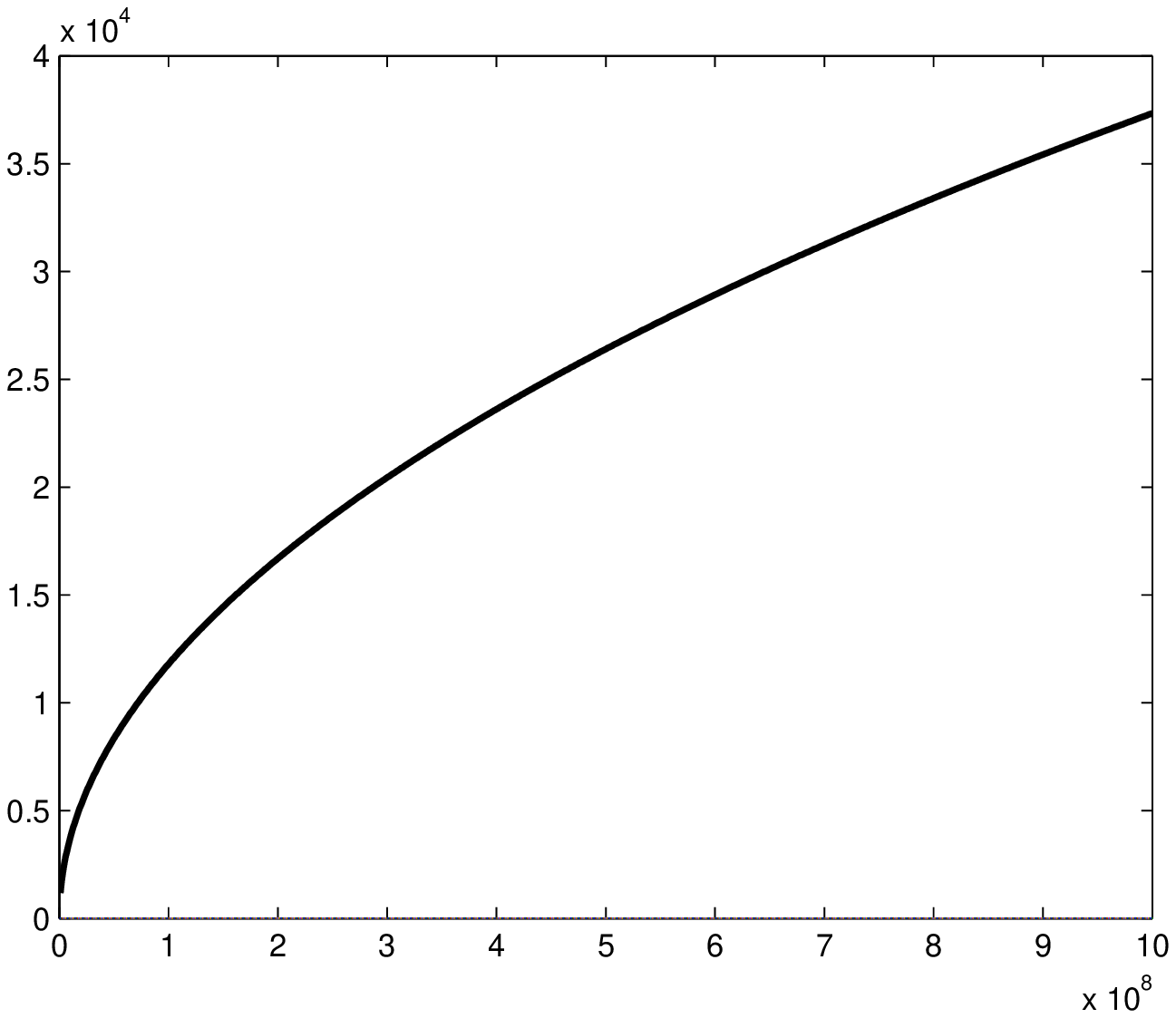}
\label{dyad_datab}
}
\caption{Sample means for $E|\F\setminus\L|$ versus $n$.  We can see the expected logarithmic and half-power growth in (a) and (b) respectively.  The dotted lines indicate the best fit curves described in this section.  In (b), the best fit curve is too closely aligned with the experimental data to be visible.}
\label{dyad_data}
\end{figure}

We next looked to find criteria that induce domains other than boxes in order to test Theorem \ref{thm:asym}.  A somewhat contrived example involves the criteria $|\Delta x| + |\Delta y|$ and $|\Delta x| - |\Delta y|$, which, when applied to uniformly sampled data on $[0,1]^2$, yields dyads sampled on a diamond domain, as shown in Figure \subref*{dyad_samplesb}.   In this case, Theorem \ref{thm:asym} suggests that $\F\setminus\L$ should grow as $\sqrt{n}$.  Figure \subref*{dyad_datab} shows the sample means versus number of dyads and a best fit curve of the form $y = \alpha n^\beta$.  A linear regression on $\ln y$ versus $\ln n$ gave $\alpha=1.1642$ and $\beta = 0.5007$.  Although this example may not be practical, it is simply meant to illustrate the applicability of Theorem \ref{thm:asym} for non-independent samples. In each experiment, we varied the number of dyads between $10^6$ to $10^9$ in increments of $10^6$ and computed the size of $\F\setminus \L$ after each increment.  We ran each experiment $1,000$ times to compute the sample means shown in Figure \ref{dyad_data}.  

\section{Implementation of PDA anomaly detector}
\label{implement}
Pseudocode for the PDA anomaly detector was presented as Algorithm \ref{alg} in 
Section \ref{score}. 
The training phase involves creating $N \choose 2$ dyads 
corresponding to all pairs of training samples. 
Computing all pairwise dissimilarities in each criterion requires 
$O(mKN^2)$ floating-point operations (flops), where $m$ denotes the 
number of dimensions involved in computing a dissimilarity. 
The Pareto fronts are constructed by non-dominated sorting. 
In Section \ref{FastSort} we present a fast algorithm for non-dominated 
sorting in two criteria; for more than two criteria, we use the 
non-dominated sort of \citet{Deb2000} that constructs all 
of the Pareto fronts using $O(KN^4)$ comparisons in the worst case.

The testing phase 
involves creating dyads between the test sample and the $k_l$ nearest 
training samples in criterion $l$, which requires $O(mKN)$ flops. 
For each dyad $D_i^\text{new}$, we need to calculate the depth $e_i$. 
This involves comparing the test dyad with training dyads on 
multiple fronts until we find a training dyad that is dominated by the 
test dyad. 
$e_i$ is the front that this training dyad is a part of. 
Using a binary search to select the front and another binary search 
to select the training dyads within the front to compare to, 
we need to make $O(K\log^2 N)$ comparisons (in the worst case) to 
compute $e_i$. 
The anomaly score is computed by taking the mean of the $s$ $e_i$'s 
corresponding to the test sample; the score is then compared against a 
threshold $\sigma$ to determine whether the sample is anomalous. 
As mentioned in the Section \ref{score}, both the training and testing phases scale 
linearly with the number of criteria $K$. 

\subsection{Fast non-dominated sorting for two criteria}
\label{FastSort}
We present here a fast algorithm for non-dominated sorting in two criteria.  The standard algorithm of \citet{Deb2000} takes $O(n^2)$ time and requires $O(n^2)$ memory, 
where $n={N \choose 2}$ is the number of dyads. 
In our experience, the memory requirement is the largest obstacle to applying Pareto methods to large data sets.  Our algorithm runs in $O(n^{3/2})$ time on average and requires $O(n)$ memory.  It is based on the following observation: if the data set is sorted in ascending order in the first criterion, then the first point is Pareto-optimal, and each subsequent Pareto-optimal point can be found by searching for the next point in the sorted list that is not dominated by the most recent addition to the Pareto front.  For two criteria, there are on average $O(\sqrt{n})$ Pareto fronts, and finding each front with this algorithm requires visiting at most $n$ points, hence the $O(n^{3/2})$ average complexity.  The worst case complexity is $O(n^2)$ occurring when each Pareto front consists of a single point.  Pseudocode for the algorithm is shown in Algorithm \ref{fast_ndom}.  It has recently come to our attention that an $O(n\ln n)$ algorithm exists for the canonical anti-chain partition problem~\cite{felsner1999}, which is equivalent to non-dominated sorting in two criteria, and can also be used to quickly construct the Pareto fronts.

\begin{algorithm}[t]
\caption{Fast non-dominated sorting.}
\label{fast_ndom}
\begin{algorithmic}[1]
   \REQUIRE{Arrays $X$ and $Y$ of length $n$ (the values of the two criteria)}
   \STATE {Sort $X$ and $Y$ according to $X$ in ascending order}
	\WHILE { $X$ and $Y$ are nonempty}
		\STATE {Add $(X(1),Y(1))$ to current Pareto front}
      \STATE {$y\leftarrow Y(1)$}
      \FOR {$i=2\to{\rm length}(X)$}
		   \IF {$Y(i) \leq y$}
            \STATE {Add $(X(i),Y(i))$ to current Pareto front}
            \STATE {$y \leftarrow Y(i)$}
         \ENDIF
	   \ENDFOR
      \STATE {Remove current Pareto front from $X$,$Y$}
   \ENDWHILE
\end{algorithmic}
\end{algorithm}

\subsection{Selection of parameters}
\label{param}
The parameters to be selected in PDA are $k_1, \ldots, k_K$, which 
denote the number of nearest neighbors in each criterion. 
We connect each test sample $X$ 
to a training sample $X_j$ if $X_j$ is one of the $k_i$ nearest neighbors of 
$X$ in terms of the dissimilarity measure defined by criterion $i$. 
We now discuss how these parameters $k_1, \ldots, k_K$ can be selected. 
For simplicity, first assume that there is only one criterion, so that a 
single parameter $k$ is to be selected. 
PDA is able to detect an anomaly if the distribution of its dyads with 
respect to the Pareto fronts differs from that of a nominal sample. 
Specifically the mean of the depths of the dyads (the $e_i$'s) corresponding 
to an anomalous sample must be higher than that of a nominal sample. 
If $k$ is chosen too small, this may not be the case, especially if there 
are training samples present near an anomalous sample, in which case, 
the dyads corresponding to the anomalous sample may reside near shallow 
fronts much like a nominal sample. 
On the other hand, if $k$ is chosen too large, many dyads may correspond 
to connections to training samples that are far away, even if the test 
sample is nominal, which also makes the mean depths of nominal and anomalous 
samples more similar.

We propose to use the properties of $k$-nearest neighbor graphs ($k$-NNGs) 
constructed on the training samples to select the number of training samples 
to connect to each test sample. 
We construct symmetric $k$-NNGs, i.e.~we connect samples $i$ and $j$ if 
$i$ is one of the $k$ nearest neighbors of $j$ or $j$ is one of the $k$ 
nearest neighbors of $i$. 
We begin with $k=1$ and increase $k$ until the $k$-NNG of the 
training samples is connected, i.e.~there is only a single connected 
component. 
By forcing the $k$-NNG to be connected, we ensure that there are no 
isolated regions of training samples. 
Such isolated regions could possibly lead to dyads corresponding to anomalous 
samples residing near shallow fronts like nominal samples, 
which is undesirable. 
By keeping $k$ small while retaining a connected $k$-NNG, we are trying to 
avoid the problem of having too many dyads so that even a nominal sample may 
have many dyads located near deep fronts. 
This method of choosing $k$ to retain connectivity has been used as a 
heuristic in other unsupervised learning problems, such as spectral clustering 
\citep{vonLuxburg2007}. 
Note that by requiring the $k$-NNG to be connected, we are implicitly assuming 
that the training samples consist of a single class or multiple classes 
that are in close proximity. 
If the training samples contain multiple well-separated classes, 
such an approach may not work well. 

Now let's return to the situation PDA was designed for, with $K$ different 
criteria. 
For each criterion $i$, we construct a $k_i$-NNG using the corresponding 
dissimilarity measure and increase $k_i$ until the $k_i$-NNG is connected. 
We then connect each test sample to $s = \sum_{i=1}^K k_i$ training 
samples. 
Note that we are choosing each $k_i$ independent of the other criteria, 
which is probably not an optimal approach. 
In principle, an approach that chooses the $k_i$'s jointly could perform 
better; however, such an approach would add to the complexity.
We choose \emph{separate} $k_i$'s for each criterion, 
which we find is necessary to obtain good performance when different 
dissimilarities have varying scales and properties. 
There are, however, pathological examples where the independent approach 
could choose $k_i$'s poorly, such as the well-known example of two moons. 
These examples typically involve multiple well-separated classes, which may be 
problematic as previously mentioned. 
How to choose the $k_i$'s when the training samples contain multiple 
well-separated classes is beyond the 
scope of this paper and is an area for future work. 
We find the proposed heuristic to work well in practice, including for both 
examples presented in Section \ref{exp}.

\section{Additional discussion on pedestrian trajectories experiment}
\label{traj_add}
\begin{figure}[t]
\begin{center}
\mbox{	
  \includegraphics[width=5 cm, trim=0.5cm 0cm 0.5cm 0.7cm, clip=true]{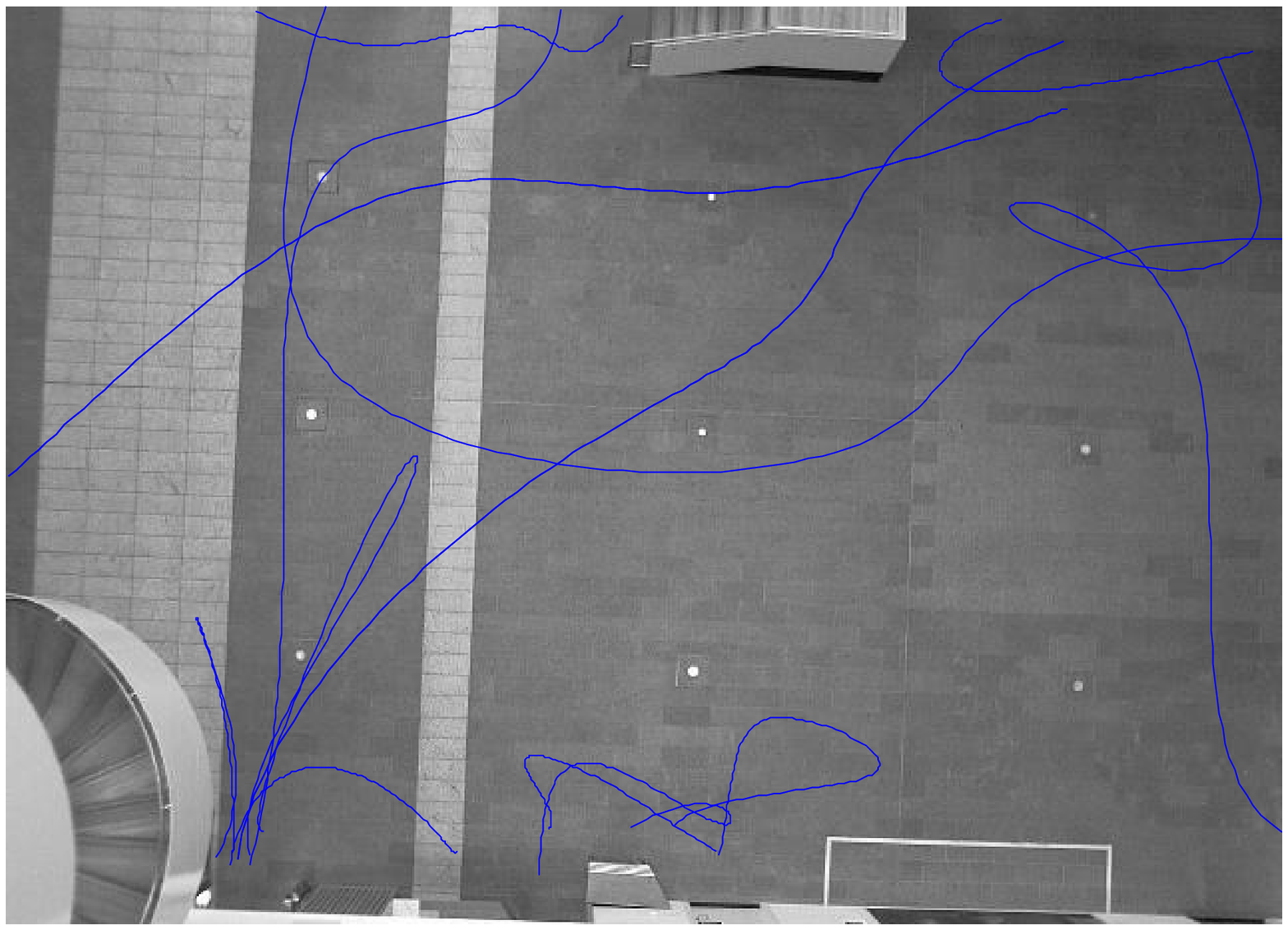}
  \includegraphics[width=5 cm, trim=0.5cm 0cm 0.5cm 0.7cm, clip=true]{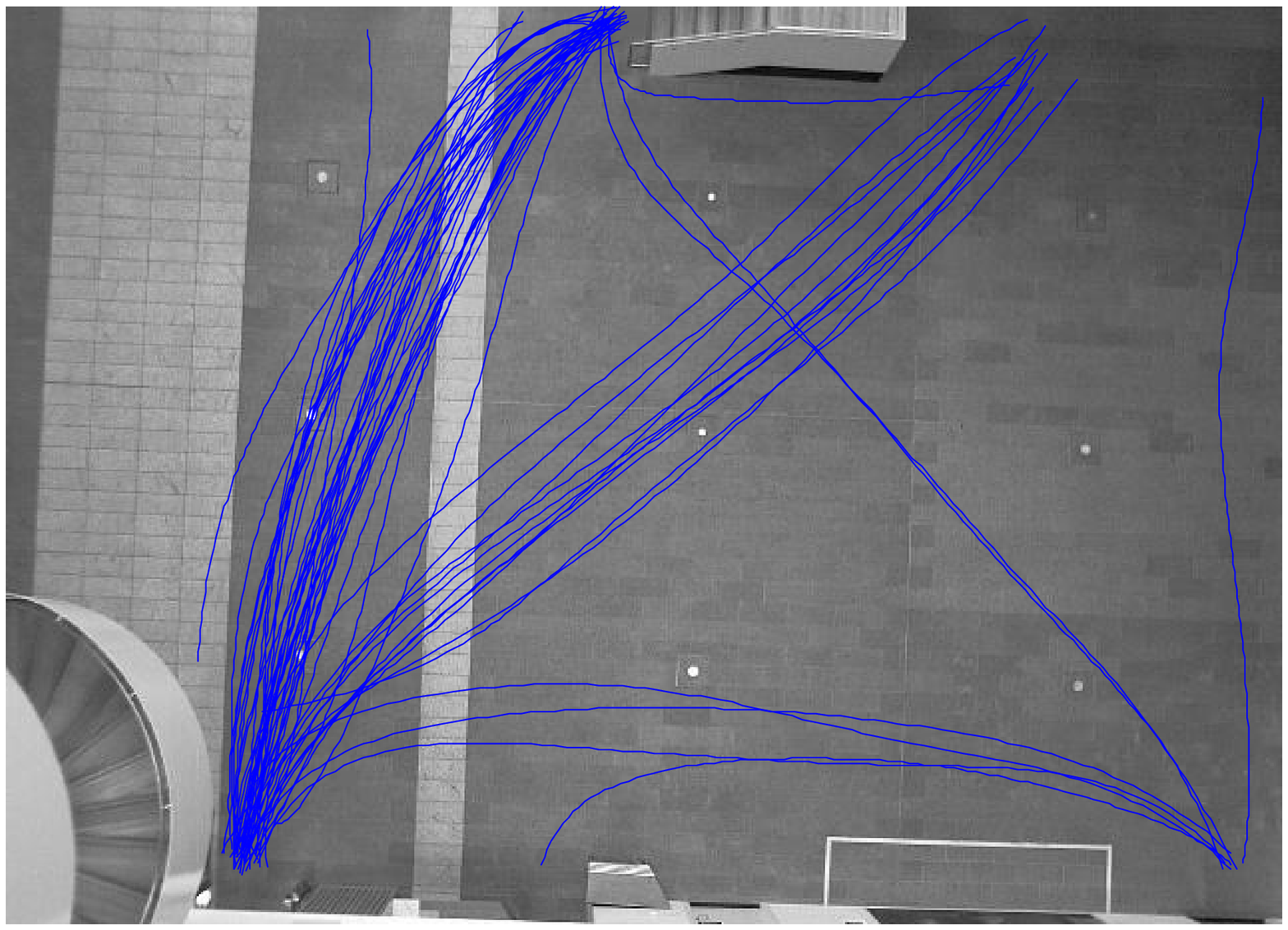}
  }
  \caption{\emph{Left:} Some abnormal trajectories detected by PDA method. 
  \emph{Right:} Trajectories with relatively low anomaly scores.  }
  \label{traj_PDA}
\end{center}
\end{figure}

\begin{figure}[t]
\begin{center}
  \includegraphics[width=8 cm]{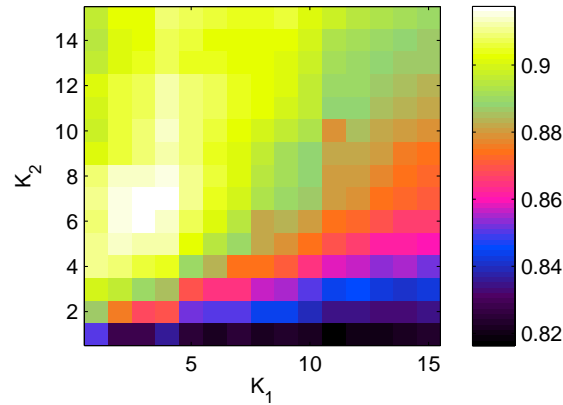}
  \caption{AUCs for different choices of $[k_1,k_2]$. 
  The automatically selected parameters $[k_1=3,k_2=6]$ are very close to the 
  optimal parameters $[k_1=4,k_2=7]$.}
  \label{auc_heatmap}
  \end{center}
\end{figure}

Figure \ref{traj_PDA} shows some abnormal trajectories and nominal trajectories
detected using PDA. 
Recall that the two criteria used are walking speed and trajectory shape. 
Anomalous trajectories could have anomalous speeds or shapes 
(or both), so some anomalous trajectories in Figure \ref{traj_PDA} may not 
look anomalous by shape alone. 
We find that the heuristic proposed in Section \ref{param} for choosing the 
$k_i$'s performs quite well in this experiment, as shown in Figure 
\ref{auc_heatmap}. 
Specifically, the AUC obtained when using the parameters chosen by the proposed 
heuristic is very close to the AUC obtained when using the optimal parameters, 
which are not known in advance. 
As discussed in Section \ref{traj}, it is also higher than the 
AUCs of all of the single-criterion anomaly detection 
methods, even under the best choice of weights.

\end{appendices}

\end{document}